\documentclass[twoside,11pt]{article}

\usepackage[abbrvbib, preprint]{jmlr2e}

\usepackage{amsmath}
\usepackage[ruled, lined, linesnumbered, commentsnumbered, longend]{algorithm2e}
\usepackage{xcolor}
\usepackage{booktabs}
\usepackage{lscape}
\usepackage{bbm}
\usepackage{url}

\SetCommentSty{mycommfont}
\SetKwFor{Loop}{Loop}{}{EndLoop}

\ShortHeadings{High-dim Batch BO without Reconstruction Mappings}{Horiguchi, Iwata, Tsuzuki and Ozawa}
\firstpageno{1}

\begin{document}

\title{Linear Embedding-based High-dimensional Batch Bayesian Optimization without Reconstruction Mappings}

\author{\name Shuhei A. Horiguchi \email shuhei-horiguchi@g.ecc.u-tokyo.ac.jp \\
       \addr Graduate School of Information Science and Technology\\
       The University of Tokyo\\
       Tokyo, 113-8656, Japan
       \AND
       \name Tomoharu Iwata \email tomoharu.iwata.gy@hco.ntt.co.jp \\
       \addr NTT Communication Science Laboratories\\
       Kyoto, 619-0237, Japan
       \AND
       \name Taku Tsuzuki \email tsuzuki@epistra.jp \\
       \name Yosuke Ozawa \email ozawaysk@epistra.jp \\
       \addr Epistra Inc.\\
       Tokyo, 105-0013, Japan}

\editor{}

\maketitle

\begin{abstract}
The optimization of high-dimensional black-box functions is a challenging problem.
When a low-dimensional linear embedding structure can be assumed, existing Bayesian optimization (BO) methods often transform the original problem into optimization in a low-dimensional space. They \textit{exploit} the low-dimensional structure and reduce the computational burden.
However, we reveal that this approach could be limited or inefficient in \textit{exploring} the high-dimensional space mainly due to the biased reconstruction of the high-dimensional queries from the low-dimensional queries.
In this paper, we investigate a simple alternative approach: tackling the problem in the original high-dimensional space using the information from the learned low-dimensional structure.
We provide a theoretical analysis of the exploration ability.
Furthermore, we show that our method is applicable to batch optimization problems with thousands of dimensions without any computational difficulty. We demonstrate the effectiveness of our method on high-dimensional benchmarks and a real-world function.
\end{abstract}


\section{Introduction}\label{sec:intro}

Because of its sample efficiency, Bayesian optimization (BO) has become a popular method for the global optimization of black-box functions. It has been successfully applied to a variety of fields ~\citep{Snoek2012}.
In real-world experiments, the function often has a large number of variables. Standard BO approaches are known to be limited to problems with moderate dimensionality, e.g., the dimensionality should be less than about ten~\citep{Wang2016}.
To tackle the high-dimensional problems, previous studies have assumed effective low-dimensionality of the function. One of the assumptions is linear embedding: all the variation of the function is captured by a linear subspace with dimensionality lower than the original dimensionality.
As~\cite{Djolonga2013} proved, if the error of learning the true low-dimensional structure can be controlled, one can achieve the regret bound with sub-exponential dependence on the original dimensionality. Several studies showed the effectiveness of the learning linear embedding approach in various cases~\citep{Zhang2019,Chen2020}.
Unlike the learning embedding approach, one can use a fixed random embedding~\citep{Wang2016,Binois2020,Letham2020}. The random embedding approach is effective in extremely high-dimensional problems~\citep{Wang2016}.

In this paper, we reveal several drawbacks in the standard technique in high-dimensional BO literature where the original high-dimensional problem is transformed into a low-dimensional problem. With this technique, one has to map query points in the low-dimensional space into the original high-dimensional space (\textit{reconstruction mapping}) to observe the function values at these points. Existing methods use a deterministic function, e.g., the pseudo-inverse matrix of the embedding matrix~\citep{Djolonga2013,Chen2020}. 
However, one point in the low-dimensional space corresponds to infinitely many points in the high-dimensional space. If the embedding matrix does not exactly capture the true low-dimensional structure, the fixed deterministic reconstruction limits the space to explore.

To overcome the issues, we propose a high-dimensional Bayesian optimization method that avoids the problems. We use the Mahalanobis kernel for Gaussian Processes to learn the effective low-dimensional structure of the data. The proposed method finds query points directly in the high-dimensional space using the information of the learned low-dimensional structure. It is equivalent to finding query points in the low-dimensional space and mapping each point to a randomly picked point in the high-dimensional corresponding subspace. 
Compared with this two-step procedure, our one-step method is simple to implement and not computationally intensive.

In addition, our method is applicable to batch optimization problems using Determinantal Point Processes (DPPs). 
Several recent studies considered batch BO for high-dimensional functions~\citep{Wang2017,Wang2018,Eriksson2019}.
While \texttt{TuRBO} algorithm~\citep{Eriksson2019} do not assume low-dimensional structure of the function, \cite{Wang2017} and~\cite{Wang2018} assumed additive structure. The linear embedding assumption has not been investigated in the batch BO literature.

We performed numerical experiments on several benchmark functions with thousands of dimensions and a real-world problem. The results show that our proposed method outperforms the state-of-the-art when the function satisfies the linear embedding assumption. The source code is available at \url{https://github.com/s-horiguchi/MahalanobisBatchBO}.

\section{High-dimensional batch BO}
To begin with, we introduce the batch optimization problem in the high-dimensional setting in Section~\ref{sec:problem}.
Then we describe the Bayesian optimization method for the sequential optimization problem in Section~\ref{sec:bo_with_gp} and the batch optimization problem in Section~\ref{sec:dpp_batch}.

\subsection{Problem formulation}\label{sec:problem}
We consider a minimization problem of an unknown function $f$ over a $D$-dimensional box $\mathcal{X}\subset \mathbb{R}^D$ where $D \gg 10$. $f$ is assumed to have a low-dimensional structure $f(x) = \tilde{f}(Ax)$, where $A$ is a $d_{\mathrm{true}}\times D$ linear embedding matrix with $d_{\mathrm{true}} \ll D$.

We are given the data of already observed points $\mathcal{D}_0 = \{ (x_\tau, y_\tau) \}_{\tau=1}^{N_{\mathrm{init}}}$. At each time step $t=1,2,\ldots,T$, we select a batch of $N_{\mathrm{batch}}$ points $\{ x_{t,b} \}_{b=1}^{N_{\mathrm{batch}}} \in \mathcal{X}^{N_{\mathrm{batch}}}$ and observe a possibly noisy function value $y_{t,b} = f(x_{t,b}) + \epsilon_{t,b} $ for $b=1,2,\dots,N_{\mathrm{batch}}$, where $\epsilon_{t,b}$ is i.i.d. Gaussian noise with mean zero and variance $\sigma^2$. The set of observations up to time step $t$ is denoted as $\mathcal{D}_t = \{ (x_\tau, y_\tau) \}_{\tau=1}^{N_{\mathrm{init}}} \cup \{(x_{s,b},y_{s,b}) \}_{s=1,b=1}^{t,N_{\mathrm{batch}}}$. For simplicity, we also denote $\mathcal{D}_t = \{ (x_\tau, y_\tau) \}_{\tau=1}^{N_t}$, where $N_t = N_{\mathrm{init}} + t N_{\mathrm{batch}}$.

\subsection{BO with Gaussian Processes}\label{sec:bo_with_gp}
Bayesian optimization uses a surrogate model of the function $f$. At each time step $t$, the next query points are determined from the posterior distribution of $f$ given data.

A popular choice of the model is a Gaussian process prior $f \sim \mathcal{GP}(0, \kappa)$ with covariance (kernel) function $\kappa : \mathcal{X}\times\mathcal{X}\rightarrow\mathbb{R}_+$.
Given observations $\mathcal{D}_{t}$, the posterior distribution of $f$ is also a Gaussian process~\cite{Rasmussen2006} as $f\vert \mathcal{D}_t \sim \mathcal{GP}(\mu_t, \kappa_t)$, with the following posterior mean and covariance function:
\begin{equation}\label{eq:posterior_mean_cov}
	\begin{aligned}
	\mu_t(x)&=\mathbf{k}_t(x)^\top (K_t+\sigma^2I)^{-1}\mathbf{y}_t, \\
	\kappa_t(x,x^\prime)&=\kappa(x,x^\prime)-\mathbf{k}_t(x)^\top(K_t+\sigma^2I)^{-1}\mathbf{k}_t(x^\prime),
\end{aligned}
\end{equation}
where $[K_t]_{i,j} = \kappa(x_i, x_j)$, $[\mathbf{k}_t(x)]_i = \kappa(x_i, x)$ and $[\mathbf{y}_t]_i = y_i$ for $i,j=1,\dots,N_{t}$.

In the sequential setting ($N_\mathrm{batch}=1$), the next query point is typically determined as the optimum of an acquisition function, which evaluates the utility of choosing that point. For example, Lower Confidence Bound (LCB) is one of the popular acquisition functions~\citep{Snoek2012}. Given the posterior mean $\mu_{t-1}$ and the posterior variance $\sigma_{t-1}(x) := \kappa_{t-1}(x, x)$, the query point $x_{t}$ is the minimizer of 
\begin{equation}\label{eq:LCB}
	\alpha_{t}^{(-\beta_t)}(x) := \mu_{t-1}(x) - {\beta_t} \sigma_{t-1}(x),
\end{equation}
where $\beta_t$ is a tunable parameter to balance exploitation and exploration. EST is a variant of LCB, where $\beta_t$ is adaptively tuned so that $x_{t}$ is most likely to achieve the lowest function value~\citep{Wang2016a}.

\subsection{DPPs-based batch selection}\label{sec:dpp_batch}

In the batched setting ($N_{\mathrm{batch}} > 1$), the DPP-EST-SAMPLE algorithm by~\cite{Kathuria2016} is one of the few batch BO methods with theoretical guarantees. This method determines the first query point $x_{t,1}$ by minimizing EST. The set of remaining points $\{ x_{t,b} \}_{b=2}^{N_{\mathrm{batch}}}$ is sampled from a $k$-DPP.

$k$-DPPs are probability distributions over subsets of a fixed ground set where the size of the subset is restricted to $k$~\citep{Kulesza2012}.
A $k$-DPP on a discrete domain $\Omega$ has the probability distribution for a set $S \subset \Omega$ of size $k$ as $\mathbb{P}(S) \propto \det(L_S)$, where $L$ is $\vert\Omega\vert \times \vert\Omega\vert$ kernel matrix and $L_S$ is its submatrix indexed by the corresponding elements of $S$. For continuous domain $\Omega$, the kernel matrix becomes a continuous function $L : \Omega \times \Omega \rightarrow \mathbb{R}$ and the probability density function is proportional to the determinant~\citep{pmlr-v97-rezaei19a}. DPPs can be understood as a tradeoff between quality and diversity. Roughly speaking, the diagonal element $L(x,x)$ measures the quality of an item $x\in \Omega$, and the non-diagonal element $L(x,y)$ is a signed measure of similarity between items $x,y\in\Omega$~\citep{Kulesza2012}. A DPP with this kernel $L$ tends to choose a set $S \subset \Omega$ with high quality of each item and high diversity in terms of similarity.

For the batch selection in BO, \cite{Kathuria2016} proposed to use a $(N_{\mathrm{batch}}-1)$-DPP with kernel function
\begin{equation}\label{eq:DPP_kernel_org}
	L_t(x, x^\prime) = \delta(x-x^\prime) + \sigma^{-2} \kappa_{t,1}(x,x^\prime),
\end{equation}
over the \textit{relevant region}
\begin{equation*}\label{eq:relevant_region}
\mathcal{R}_t = \left\{ x\in\mathcal{X} ~\big\vert~ \alpha_t^{(-2\beta_t)}(x) \leq \min_{x^\prime\in\mathcal{X}}\alpha_t^{(\beta_t)}(x^\prime) \right\},
\end{equation*}
where $\delta$ is the Dirac delta function, and $\kappa_{t,1}$ is the posterior covariance function after selecting $x_{t,1}$.

Considering the property of $k$-DPPs, the DPP kernel $L_{t}$ prefers the points with large posterior variance $\sigma_{t,1}(x)=\kappa_{t,1}(x,x)$. It discourages the selection of a pair of points similar to each other in the sense of posterior covariance $\kappa_{t,1}(x,x^\prime)$. Also, the relevant region works as a filter to ensure the quality of the query points. $\mathcal{R}_t$ contains $\arg\min_{x\in\mathcal{X}} \alpha_t^{(-\beta_t)}(x)$ with high probability~\citep{Contal2013}.

\section{Two-Step Method for High-Dimensional Problems and Its Challenges}\label{sec:challenges}

Previously proposed BO methods for high-dimensional problems often transform the original high-dimensional problem into a low-dimensional problem. While the original problem is to minimize $f$ over the high-dimensional domain $\mathcal{X}$, assume $f = \tilde{f} \circ g$, where $g:\mathcal{X} \rightarrow \mathcal{Z}$ and $\tilde{f}:\mathcal{Z} \rightarrow \mathbb{R}$. The embedding map $g$ is constructed as a randomly initialized fixed function, or a learned function adapted to the data so that $\mathcal{Z}$ has lower dimensionality than $\mathcal{X}$, e.g., $g(x)=Bx$ with a $d\times D$ matrix $B$.
To get the next query point $x_{t} \in \mathcal{X}$, previous methods generally followed the \textit{two-step framework}: 1)
finding $z_{t} \in \mathcal{Z}$ by running a usual BO against $\tilde{f}$, and 2) mapping it on $\mathcal{X}$ by some strategy. We summarize the procedure in Algorithm~\ref{alg:two-step_bo}.

\begin{algorithm}[ht]
\DontPrintSemicolon
\caption{Two-step method}
\label{alg:two-step_bo}

\KwIn{Input space $\mathcal{X}$, intial observations $\mathcal{D}_0$}
\KwOut{Observations $\mathcal{D}_T$}

Generate a random embedding $g$ \tcp*{For random embedding methods (Sec.\ref{sec:choice_of_embedding})}
\For{$t = 1$ \KwTo $T$}{
	Learn the embedding $g$ from $\mathcal{D}_t$ \tcp*{For learning embedding methods (Sec.\ref{sec:choice_of_embedding})}
	Fit a Gaussian process to the data $\tilde{\mathcal{D}}_t = \{(g(x), y) ~|~ (x,y) \in \mathcal{D}_t \}$\;
	Find the low-dimensional queries $\{z_{t,b}\}_{b=1}^{N_\mathrm{batch}}$ on $\mathcal{Z}$ \tcp*{Sec.\ref{sec:opt_zonotope}}
	\For{$b = 1$ \KwTo $N_\mathrm{batch}$}{
		Reconstruct the query $x_{t,b} \in \mathcal{X}$ from $z_{t,b} \in \mathcal{Z}$ \tcp*{Sec.\ref{sec:reconst}}
		Observe $y_{t,b} = f(x_{t,b}) + \epsilon_{t,b}$\;
	}
	$\mathcal{D}_{t+1} \leftarrow \mathcal{D}_t \cup \{ (x_{t,b}, y_{t,b}) \}_{b=1}^{N_\mathrm{batch}}$\;
}
\end{algorithm}

Although these methods \textit{exploit} the low-dimensionality and reduce the computational cost, there are several fundamental difficulties with this two-step method in terms of \textit{exploration}.
In the following, we discuss the difficulties and previous approaches. Section~\ref{sec:choice_of_embedding} is about choosing the embedding map $g$. While this paper mainly focuses on the method which chooses $g$ adaptively, we briefly explain the methods based on random embedding and highlight the difference in motivation between them.
In Section~\ref{sec:opt_zonotope}, we explain the difficulty in finding the low-dimensional queries on $\mathcal{Z}$.
Section~\ref{sec:reconst} is one of the major contributions of this paper, where we identify the shortcoming of existing reconstruction mappings and present a possible workaround.

Our discussion focuses on the linear embedding case, $g(x) = Bx$, but it can be applied to the nonlinear embedding case.

\subsection{Choice of embedding map}\label{sec:choice_of_embedding}
First, since one has no access to the true embedding matrix $A$, the embedding map $g$ or the matrix $B$ may differ from $A$. It is constructed as a randomly initialized fixed matrix or learned from the data.

In the random embedding methods, the target is not the optimum in the high-dimensional domain $\mathcal{X}$ but the optimum in the subset $h(\mathcal{Z})$, which is fixed throughout the trial.
Since the randomly chosen matrix $B$ has no correlation with the true embedding, the optimum of $f$ on $h(\mathcal{Z}) \in \mathcal{X}$ is, in general, different from the true optimum on $\mathcal{X}$.
Nevertheless, under some assumptions, the global optimum in $\mathcal{X}$ exists in $h(\mathcal{X})$ with good probability~\citep{Wang2016}.
The random embedding methods are computationally cheap and scale to extremely high-dimensional problems~\citep{Wang2016}.

In contrast, the learning embedding methods aim to recover the true effective subspace. 
Since the observed data $\mathcal{D}_t = \{ (x_\tau, y_\tau)\}_\tau$ contains the information about the true embedding $A$, one can adaptively tune $B$ throughout the optimization process so that the subspace spanned by $B$ is aligned with the true subspace spanned by $A$. For example, previous methods employ the low-rank matrix recovery~\citep{Djolonga2013} or the sliced inverse regression~\citep{Zhang2019,Chen2020}.

Intuitively, as long as the computational burden of learning is acceptable and the learning is successful, the learning embedding methods perform better than the random embedding methods.
While we compare these two approaches in our experiment, this work particularly focuses on the potential issues in the learning embedding methods.

\subsection{Optimization and sampling over a zonotope}\label{sec:opt_zonotope}
Next, we present the previous approaches for the first step and point out their limitations.
The first step has fundamental difficulty since it requires the optimization and sampling over $\mathcal{Z}$, which is not box-shaped.
Linear projection of $D$ dimensional box $\mathcal{X}$ onto $d$ dimensional space is a kind of polytope called zonotope~\citep{Binois2020}. The complicated shape of the zonotope makes optimization difficult. 
Algorithms developed so far cannot enumerate the zonotope vertices in a reasonable time when, e.g., $D>100$~\citep{Stinson2016}.

One can avoid the difficulty by modifying the embedding map $g$ so that the image of $g$ becomes a box~\citep{Binois2015,Binois2020,Chen2020}.
However, as~\cite{Letham2020} discussed, the modification has a detrimental effect on modeling the $\tilde{f}$.

Instead, \cite{Letham2020} proposed to optimize over a polytope
\begin{equation*}
	\mathcal{Z}^+ := \{ z \in \mathbb{R}^d | B^+ z \in \mathcal{X} \}.
\end{equation*}
Since optimization over this polytope is optimization with linear constraint, it can be handled with off-the-shelf optimization tools~\citep{Letham2020}.
However, the polytope $\mathcal{Z}^+$ is just a proper subset of $\mathcal{Z}$. Optimization over $\mathcal{Z}^+$ may not find the true optimum in $\mathcal{Z}$.

Furthermore, in the batched setting, especially when using DPPs, we need many random samples on the domain of $\tilde{f}$. 
As far as we know, there is no efficient sampler on the zonotope. Rejection sampling becomes inefficient and impractical as $D$ or $d$ grows.

\subsection{Reconstruction strategy with the erroneously learned embedding}\label{sec:reconst}
Once we have the low-dimensional query $z_q = z_{t,b}$ on $\mathcal{Z}$, we map it back to the high-dimensional domain to get the query $x_{t,b}$.
Most of the previous studies use a deterministic function $h: \mathcal{Z} \rightarrow \mathcal{X}$, including the pseudo-inverse map $h(z) = g^+(z) := B^+ z$, where $B^+:=B^\top (BB^\top)^{-1}$ is the Moore–Penrose pseudo-inverse of $B$.
The use of pseudo-inverse is implicit but common, e.g.,~\citep{Djolonga2013,Chen2020}.

Note that there are infinitely many points $x \in \mathcal{X}$ that map to the same $z \in \mathcal{Z}$:
\begin{equation*}\label{eq:inverse_set}
\begin{aligned}
	g^{-1}(z) &:=\{x\in \mathcal{X} ~|~ g(x)=z\} \\
	&= \underset{x \in \mathcal{X}}{\arg\min} ~\Vert Bx - z \Vert^2  \\
	&= \{ B^+z + w ~|~ w\in N(B) \}  \cap \mathcal{X},
\end{aligned}	
\end{equation*}
where $N(B)$ is the right null space of $B$.
If the learned matrix $B$ is equal to the true matrix $A$, every point $x \in g^{-1}(z)$ has the same value $f(x)$.
However, such a situation never happens in practice, especially in the early stage of learning.
Thus, the choice of the reconstruction strategy matters.

We propose one of the desirable properties of the reconstruction strategy using the idea of Bayesian inference. We write the probability distribution of the output $x$ given $z_q$ as $\mathbb{P}(x | z_q)$. If the strategy is deterministic, $\mathbb{P}(x | z_q) \propto \mathbbm{1}(x = h(z_q))$, where $\mathbbm{1}(A) = 1$ when $A$ is true and $0$ otherwise.

Suppose we have a prior distribution $\mathbb{P}_{prior}(x^\dagger)$ on the location of the optima $x^\dagger$ on $\mathcal{X}$. For example, $\mathbb{P}_{prior}$ is the uniform distribution on $\mathcal{X}$. The true $x^\dagger$ is unknown, but we can observe the optima embedded in the low-dimensional space $\mathcal{Z}$. By the Bayes theorem, the posterior is given by
\begin{equation}\label{eq:posterior}
	\mathbb{P}_{post}(x^\dagger | z^\dagger) 
	\propto
	\mathbbm{1}(z^\dagger = Bx^\dagger) \mathbb{P}_{prior}(x^\dagger).
\end{equation}

\begin{definition}[Unbiased reconstruction]
	A reconstruction strategy is said to be unbiased with respect to a prior $\mathbb{P}_{prior}$ if the probability distribution $\mathbb{P}(x|z_q)$ of its output $x$ given $z_q$ is equal to the posterior distribution $\mathbb{P}_{post}(x | z_q)$ in Equation~\eqref{eq:posterior} for any $z_q \in \mathcal{Z}$.
	Otherwise, the strategy is said to be biased with respect to $\mathbb{P}_{prior}$.
\end{definition}

Any deterministic reconstruction $h$ is biased if the prior $\mathbb{P}_{prior}$ has positive probability on the whole $\mathcal{X}$, e.g., $\mathbb{P}_{prior}$ is the uniform distribution.
The pseudo-inverse reconstruction $h=g^+$ is unbiased if and only if the prior $\mathbb{P}_{prior}$ put a positive probability only on the subspace spanned by $B^+$.
For general $\mathbb{P}_{prior}$, random sampling from $g^{-1}(z_q)$ is one way to remove the bias.

\begin{algorithm}[ht]
\DontPrintSemicolon
\caption{Randomized reconstruction strategy}
\label{alg:random_reconstruction}
\KwIn{Input space $\mathcal{X}$, embedding map $B$, prior distribution $\mathbb{P}_{prior}$, low-dimensional query $z_q$}
\KwOut{High-dimensional query $x$}
\Loop{}{
	Initialize $x^{(0)}$ by sampling from $\mathbb{P}_{prior}$\;
	$k \leftarrow 0$\;
	\While{$x^{(k)} \in \mathcal{X}$}{
		\If{$\frac{1}{2}\| Bx^{(k)} - z_q \|^2$ is sufficiently small}{
			\KwRet $x^{(k)}$\;
		}
		\tcp{Gradient descent with the exact line search}
		$\gamma \leftarrow \frac{\left\| B^\top \left(Bx^{(k)} - z_q \right) \right\|^2}{\left\| BB^\top \left(Bx^{(k)} - z_q \right) \right\|^2}$\;
		$x^{(k+1)} \leftarrow x^{(k)} - \gamma B^\top \left(Bx^{(k)} - z_q \right)$\;
		$k \leftarrow k + 1$\;
	}
}
\end{algorithm}

Here, we consider such a randomized algorithm in Algorithm~\ref{alg:random_reconstruction}.
Given the low-dimensional query $z_q \in \mathcal{Z}$, it runs the gradient descent to minimize $\frac{1}{2}\| B x - z_q \|^2$ from random initializations.
If the update gets out of $\mathcal{X}$, the gradient descent is restarted from a different initial point. 
We can show that the algorithm is unbiased with respect to $\mathbb{P}_{prior}$ under some assumptions.

\begin{theorem}\label{thm:randomized_unbiased}
	Suppose $B$ has full rank. The output of the Algorithm~\ref{alg:random_reconstruction} is always one of the points in $g^{-1}(z_q)$ for any $z_q \in \mathcal{Z}$.
	Furthermore, assume the following condition holds:
	\begin{equation}\label{eq:prior_condition}
		\mathbb{P}_{prior}(B^+ z + w) = \mathbb{P}_{prior}(B^+ z^\prime + w), \quad \forall w\in N(B), z,z^\prime\in\mathcal{Z}.
	\end{equation}
	Then, the reconstruction given by Algorithm~\ref{alg:random_reconstruction} is unbiased with respect to $\mathbb{P}_{prior}$.
\end{theorem}
\begin{proof}
	For any vector $x^{(k)} \in \mathcal{X}$, it has a unique expression $x^{(k)} = B^+ z^{(k)} + w^{(k)}$ with $z^{(k)} = Bx^{(k)}$ and $w^{(k)} = (I - B^+B) x^{(k)} \in N(B)$.
	
	Since the gradient $B^\top (Bx^{(k)} - z_q)$ is in the row space of $B$, only the $z$ component is updated, whereas the $w$ component remains unchanged. 
	Precisely, the algorithm is equivalent to the gradient descent against $\frac{1}{2} \| B^\top (z - z_q) \|^2$.
	Because $B$ has full rank, the $z$ component converges to $z_q$ unless $x^{(k)}$ gets out of $\mathcal{X}$.
	Thus, the output $x^* = B^+ z^* + w^*$ satisfies $z^* = z_q$, equivalently, $x^* \in g^{-1}(z_q)$.

	The conditional probability distribution of the output $x^*$ given $x^{(0)}$ and $z_q$ is given by
	\begin{equation*}
		\mathbb{P}(x^* | x^{(0)}, z_q) 
		\propto \mathbbm{1}(z^* = z_q) \mathbbm{1}(w^* = w^{(0)})
	\end{equation*}
	Using the assumption in Equation~\eqref{eq:prior_condition}, we marginalize $x^{(0)}$ to have
	\begin{equation*}
	\begin{aligned}
		\mathbb{P}(x^* | z_q) 
		&= \int_\mathcal{X} \mathbb{P}(x^* | x^{(0)}, z_q) \mathbb{P}_{prior}(x^{(0)}) dx^{(0)}\\
		&\propto \mathbbm{1}(z^* = z_q) \int_\mathcal{X} 
		\mathbbm{1}(w^* = w^{(0)}) \mathbb{P}_{prior}(B^+ z^{(0)} + w^{(0)}) dx^{(0)}\\
		&= \mathbbm{1}(z^* = z_q) \int_\mathcal{X} 
		\mathbbm{1}(w^* = w^{(0)}) \mathbb{P}_{prior}(B^+ z^* + w^{(0)}) dx^{(0)}\\
		&= \mathbbm{1}(z^* = z_q) \mathbb{P}_{prior}(x^*).
	\end{aligned}
	\end{equation*}
\end{proof}

Note that the algorithm is biased without the assumption on $\mathbb{P}_{prior}$. 
Even in this case, if $\mathbb{P}_{prior}$ has positive probability on the whole $\mathcal{X}$, the set of all possible outputs of the randomized algorithm covers $g^{-1}(z_q)$.

In contrast to the nice theoretical property, this randomized algorithm does not work well in practice because the number of restarts could be very large. To avoid restarting, the initial condition $x^{(0)}$ should be a point such that the limit point of the gradient descent $B^+ z_q + (I-B^+B)x^{(0)}$ is in $\mathcal{X}$. To sample such initial conditions becomes difficult as the dimensionality grows.

\section{One-Step Method for High-Dimensional Problems}

In this section, we investigate the {\it one-step method} for high-dimensional problems in Algorithm~\ref{alg:one-step_bo}. It is the same as the standard BO method for low-dimensional problems, except that the high-dimensional query selection uses the learned embedding.

\begin{algorithm}[ht]
\DontPrintSemicolon
\caption{One-step method}
\label{alg:one-step_bo}

\KwIn{Input space $\mathcal{X}$, intial observations $\mathcal{D}_0$}
\KwOut{Observations $\mathcal{D}_T$}

\For{$t = 1$ \KwTo $T$}{
	Learn the embedding $g$ from $\mathcal{D}_t$\;
	Fit a Gaussian process to the data $\mathcal{D}_t$\;
	Find the high-dimensional queries $\{x_{t,b}\}_{b=1}^{N_\mathrm{batch}}$ on $\mathcal{X}$ \tcp*{Use the learned $g$}
	\For{$b = 1$ \KwTo $N_\mathrm{batch}$}{
		Observe $y_{t,b} = f(x_{t,b}) + \epsilon_{t,b}$\;
	}
	$\mathcal{D}_{t+1} \leftarrow \mathcal{D}_t \cup \{ (x_{t,b}, y_{t,b}) \}_{b=1}^{N_\mathrm{batch}}$\;
}
\end{algorithm}

The one-step method could circumvent the difficulties stated in Section~\ref{sec:challenges}.
We propose a concrete one-step algorithm with linear embedding $g(x) = Bx$.
It is just an ordinally batch BO with the Mahalanobis kernel.
In Section~\ref{sec:learning_maha}, we introduce how to learn the embedding $B$ as the hyperparameter optimization of the Mahalanobis kernel.
In Section~\ref{sec:query_without_reconstruction}, we present the algorithm to find the queries directly on the high-dimensional domain $\mathcal{X}$. We also show that our method is free from the difficulties described in Section~\ref{sec:opt_zonotope} and \ref{sec:reconst}.

\subsection{Learning Mahalanobis kernel}\label{sec:learning_maha}
To capture the low-dimensionality of $f$, we use the following Mahalanobis kernel as a prior covariance function
\begin{equation}\label{eq:maha_kernel}
	\kappa(x, x^\prime) = \gamma^2 \exp\left( -(x-x^\prime)^\top B^\top B (x-x^\prime) \right),
\end{equation}
where $\gamma \in \mathbb{R}$ and $B \in \mathbb{R}^{d \times D}$ are hyperparameters. This kernel has been used for Gaussian process regression~\citep{Garnett2014} and Bayesian optimization~\citep{Letham2020}. This kernel is equivalent to the following RBF kernel for $z=Bx$
\begin{equation*}\label{eq:rbf_kernel}
	\tilde{\kappa}(z, z^\prime) = \gamma^2 \exp\left( - (z-z^\prime)^\top (z-z^\prime) \right).
\end{equation*}
Also, it is reduced to the RBF kernel with ARD when $d = D$ and $B$ is diagonal. 
We choose the hyperparameters by maximizing the marginal likelihood with Adam~\citep{Kingma15}.

Note that the transformation from the Mahalanobis kernel to the RBF kernel implies that the Gaussian process regression on $\mathcal{X}$ with the Mahalanobis kernel is equivalent to the Gaussian process regression on $\mathcal{Z}$ with the RBF kernel. 
More precisely, from Equation~\eqref{eq:posterior_mean_cov}, the posterior mean function $\mu_t$ and covariance function $\kappa_t$ of the former are equivalent to the posterior mean $\tilde{\mu}_t$ and covariance $\tilde{\kappa}_t$ of the latter:
\begin{equation}\label{eq:posterior_mean_cov_maha}
\begin{aligned}
	\mu_t(x)&=\tilde{\mathbf{k}}_t(Bx)^\top (\tilde{K}_t+\sigma^2I)^{-1}\mathbf{y}_t = \tilde{\mu}_t(Bx), \\
	\kappa_t(x,x^\prime)&=\tilde{\kappa}(Bx,Bx^\prime)-\tilde{\mathbf{k}}_t(x)^\top(\tilde{K}_t+\sigma^2I)^{-1}\tilde{\mathbf{k}}_t(x^\prime) = \tilde{\kappa}_t(Bx, Bx^\prime),
\end{aligned}
\end{equation}
where $[\tilde{K}_t]_{i,j} = \tilde{\kappa}(Bx_i, Bx_j)$, $[\tilde{\mathbf{k}}_t(x)]_i = \tilde{\kappa}(Bx_i, Bx)$ for $i,j=1,\dots,N_{t}$.

\subsection{Query selection without reconstruction mapping}\label{sec:query_without_reconstruction}

Given $B$, we could consider the problem in the low-dimensional space $\mathcal{Z}$ and select query points in two steps. However, as stated in Section~\ref{sec:challenges}, this approach has several problems.
We propose to find queries in one step on the original high-dimensional space $\mathcal{X}$.
In this section, we present a variant of the DPP-EST-SAMPLE algorithm by~\cite{Kathuria2016} applied to the high-dimensional problem and show how it avoids the issues in a simple way.

\begin{algorithm}[ht]
\DontPrintSemicolon
\caption{Acquisition optimization}
\label{alg:acq_opt}

\KwIn{Input space $\mathcal{X}$, acquisition function $\alpha$, number of restart $m$.}
\KwOut{Query point $x$}

$S \leftarrow \emptyset$\;
\For{$i=1$ \KwTo $m$}{
	$k \leftarrow 0$\;
	Initialize $x^{(0)}$ by sampling from $\mathbb{P}_{prior}$\;
	\Repeat{$x^{(k)} \notin \mathcal{X}$ or $\| \nabla \alpha(x^{(k)}) \|$ is sufficiently small}{
		Determine an appropriate step size $\gamma$\;
		$x^{(k+1)} \leftarrow x^{(k)} - \gamma \nabla \alpha(x^{(k)})$\;
		$k \leftarrow k+ 1$\;
	}
	\If{$x^{(k)} \in \mathcal{X}$}{
		$S \leftarrow S \cup \{ x^{(k)} \}$\;
	}
}
\KwRet $\arg\min_{x \in S} \alpha(x)$
\end{algorithm}

The first point $x_{t,1}$ is selected by optimizing the EST acquisition function $\alpha_{t}^{(-\beta_t)}(x)$ over $\mathcal{X}$. The acquisition function is given by Equation~\eqref{eq:LCB} with the learned Mahalanobis kernel.
If we optimize the function by the gradient descent starting from random initial points generated from $\mathbb{P}_{prior}$ on $\mathcal{X}$, it is equivalent to the optimization on $\mathcal{Z}$ and the reconstruction of query on $\mathcal{X}$.
Furthermore, we will show that, in ideal situations, the optimization on $\mathcal{X}$ implies the unbiased reconstruction with respect to $\mathbb{P}_{prior}$.

The key point is that the acquisition function $\alpha_t^{(-\beta_t)}(x)$ can be written as a function of $Bx$ as
\begin{equation*}
	\alpha_t^{(-\beta_t)}(x) = \tilde{\mu}_{t-1}(Bx) - \beta_t \tilde{\sigma}_{t-1}(Bx) =: \tilde{\alpha}_t(Bx),
\end{equation*}
where we used the property of the Mahalanobis kernel in Equation~\eqref{eq:posterior_mean_cov_maha} and $\tilde{\sigma}_{t-1}(z) = \tilde{\kappa}_{t-1}(z, z)$ is the posterior variance with the RBF kernel.
In fact, $\tilde{\alpha}_t(z)$ is the EST acquisition function on $\mathcal{Z}$ where the function is modeled with RBF kernel.
Algorithm~\ref{alg:acq_opt} is the most straightforward procedure of optimizing the function $\alpha(x) := \alpha_t^{(-\beta_t)}(x)$ with the vanilla gradient descent, where we can obtain the following theorem.

\begin{theorem}\label{thm:equiv_one_two}
	Let $\alpha(x) = \tilde{\alpha}(Bx)$ be the EST acquisition function of a Gaussian process with a Mahalanobis kernel.
	Suppose $B$ has full rank and $\tilde{\alpha}(z)$ has a unique maximizer $z^\dagger \in \mathcal{Z}$ such that $\nabla \tilde{\alpha}(z^\dagger) = 0$ holds.
	Then, the output $x$ of Algorithm~\ref{alg:acq_opt} with $m\rightarrow \infty$ satisfies $Bx = z^\dagger$, and the probability distribution of the output is the same as that of randomized reconstruction (Algorithm~\ref{alg:random_reconstruction}) with $z_q = z^\dagger$.
\end{theorem}
\begin{proof}
	As in~\cite{Kim2020}, it is easy to show that $\tilde{\alpha}$ has a Lipschitz continuous gradient over $\mathcal{Z}$ with some constant $L$. Furthermore, the gradient of $\alpha$ is $(\| B \|L)$-Lipschitz continuous over $\mathcal{X}$ because for any $x,y \in \mathcal{X}$,
	\begin{equation*}
	\begin{aligned}
		\| \nabla \alpha(x) - \nabla \alpha(y) \| 
		&= \| B^T \left(\nabla \tilde{\alpha}(Bx) - \nabla \tilde{\alpha}(By) \right)\|\\
		&\leq \|B \|\cdot L \| x- y \|.
	\end{aligned}
	\end{equation*}
	Therefore, with a step size $\gamma < \frac{2}{\|B\| L}$, the gradient $\nabla \alpha(x^{(k)})$ converges to zero~\citep{Polyak1987} unless $x^{(k)}$ gets out of $\mathcal{X}$.
	The output $x^*$ of Algorithm~\ref{alg:acq_opt} satisfies $\nabla \alpha(x^*) = B^\top \nabla \tilde{\alpha}(Bx^*) = 0$. As $B$ has full rank, $\nabla \tilde{\alpha}(Bx^*) = 0$ holds, i.e. the $z$ component of $x^*$ should be a stationary point of $\tilde{\alpha}$. 
	In the limit $m\rightarrow \infty$, $x^*$ is a global minimum of $\alpha$, whose $z$ component is the unique minimizer $Bx^* = z^\dagger$. 
	On the other hand, following the same argument in the proof of Theorem~\ref{thm:randomized_unbiased}, the $w$ component of $x^*$ is the same as that of $x^{(0)}$. Thus, its distribution is precisely the same as in the randomized reconstruction.
\end{proof}

Practically, we use the L-BFGS-B~\citep{Byrd1995}, a quasi-Newton method, instead of the vanilla gradient descent.
Since the second-order methods could disturb the $w$ component, their behavior would differ from the above analysis.
However, they are much more efficient than vanilla gradient descent and work well in practice.
We also note that since the domain $\mathcal{X}$ is a box rather than a zonotope, we can use a wide variety of gradient-based optimizers.

For the remaining query points $\{ x_{t,b} \}_{b=2}^{N_{\mathrm{batch}}}$, they are sampled from $(N_{\mathrm{batch}}-1)$-DPP with kernel function
\begin{equation*}\label{eq:DPP_kernel}
	L_t(x, x^\prime) = \kappa_{t,1}(x,x^\prime),
\end{equation*}
over the relevant region $\mathcal{R}_t \subset \mathcal{X}$. Here, our DPP kernel is slightly different from Equation~\eqref{eq:DPP_kernel_org}. It is the small noise limit $\sigma \rightarrow 0$ of the original definition. This modification allows us a simple theoretical analysis.

Let us define a low-dimensional version of the DPP kernel $\tilde{L}_t(z, z^\prime) = \tilde{\kappa}_{t,1}(z, z^\prime)$ for any $z,z^\prime \in \mathcal{Z}$ and the relevant region $\tilde{\mathcal{R}}_t = \{ Bx ~|~ x \in \mathcal{R}_t \} \subset \mathcal{Z}$.
Consider a two-step DPP sampling: sample $\{ z_{t,b} \}_{b=2}^{N_{\mathrm{batch}}}$ from $(N_{\mathrm{batch}}-1)$-DPP with kernel $\tilde{L}$ over $\tilde{\mathcal{R}}_t$, and use an unbiased reconstruction strategy to get $\{ x_{t,b} \}_{b=2}^{N_{\mathrm{batch}}}$. Then, the two-step and one-step sampling are equivalent in the following sense.

\begin{theorem}
	The probability distribution of the output of the one-step DPP sampling is the same as that of the two-step DPP sampling with uniform $\mathbb{P}_{prior}$ on $\mathcal{R}_t$.
\end{theorem}
\begin{proof}
	Let $n := N_{\mathrm{batch}} - 1$ and $L_t(\{ x_{b} \}_{b=1}^{n})$ be a $n \times n$ matrix with $L_t(x_i, x_j)$ in the $(i,j)$ component.
	The probability distribution of the output in the one-step sampling is, for any $x_b \in \mathcal{R}_t$, $b=1, ...,n$,
	\begin{equation*}
		\mathbb{P}\left(\{ x_{b} \}_{b=1}^{n}\right)
		\propto \det L_t(\{ x_{b} \}_{b=1}^{n}) 
		\propto \det \tilde{L}_t(\{ Bx_{b} \}_{b=1}^{n}),
	\end{equation*}
	where we used $L_t(x, x^\prime) = \tilde{L}_t(Bx, Bx^\prime)$ for any $x,x^\prime \in \mathcal{X}$.
	
	On the other hand, the probability distribution of the output in the two-step sampling is, for any $x_b \in \mathcal{R}_t$, $b=1, ...,n$,
	\begin{equation*}
	\begin{aligned}
		\mathbb{P}\left(\{ x_{b} \}_{b=1}^{n}\right)
		&\propto \int_{\tilde{\mathcal{R}_t}} \det \tilde{L}_t(\{ z_b \}_{b=1}^{n}) \prod_{b=1}^{n} \mathbb{P}_{post}(x_b | z_b) dz_b\\
		&\propto \int_{\tilde{\mathcal{R}_t}} \det \tilde{L}_t(\{ z_b \}_{b=1}^{n}) \prod_{b=1}^{n} \mathbbm{1}(z_b = B x_b) \mathbb{P}_{prior}(x_b) dz_b\\
		&\propto \det \tilde{L}_t(\{ B x_b \}_{b=1}^{n}),
	\end{aligned}
	\end{equation*}
	where we used the assumption of uniform $\mathbb{P}_{prior}$ in the last line.
	
	Finally, consider the output $\{x_b \}_{b=1}^n$ where there exists $a \in \{ 1,..., n \}$ such that $x_a \notin \mathcal{R}_t$. The probability of such outputs is zero in both the one-step and the two-step sampling.
\end{proof}

Since the original dimensionality $D$ is high, our algorithm should have a low computational cost in $D$.
In acquisition function optimization, while we optimize over $D$ dimensional domain, each update in L-BFGS-B requires only $\mathcal{O}(D)$ operations~\citep{Byrd1995}.
In the DPP sampling step, we adopt the Gibbs sampling scheme by~\cite{pmlr-v97-rezaei19a}, which is also efficient in the high-dimensional continuous domain. They proved that the expected number of Gibbs sampling steps for mixing is the only polynomial in $N_{\mathrm{batch}}-1$ and does not depend on the dimensionality $D$. Moreover, they empirically showed that $\mathcal{O}((N_{\mathrm{batch}}-1)^2)$ is sufficient for typical cases. 
We provide a table of the running times of our experiments in the Appendix.

\section{Experiments}

To evaluate our proposed method, we performed sequential and batch BO experiments for high-dimensional versions of benchmark functions in Section~\ref{sec:benchmark} and a more realistic rover trajectory optimization problem in Section~\ref{sec:rover}.

\subsection{Benchmark functions}\label{sec:benchmark}

First, we constructed high-dimensional versions of the following five benchmark functions: Branin ($d_{\mathrm{true}}=2$), Colville ($d_{\mathrm{true}}=4$), Goldstein-Price ($d_{\mathrm{true}}=2$), Hartmann6 ($d_{\mathrm{true}}=6$), Six-Hump Camel ($d_{\mathrm{true}}=2$)\footnote{\url{http://www.sfu.ca/~ssurjano}}.
We linearly projected the function $\tilde{f}$ defined on the $d_{\mathrm{true}}$ dimensional box domain onto $D=1000$ dimensional box $\mathcal{X}=[-1,+1]^D$ with non-orthogonal non-axis-aligned random embedding $A$.
Generating 20 different $A$, we ran several BO methods one trial for each embedding $A$. $N_{\mathrm{init}}=10$ initial points were generated by Sobol sequences and shared across all trials except for some methods that require special initialization strategies. The batch size $N_{\mathrm{batch}}$ was set to 1 (sequential optimization problem) or 5 (batch optimization problem).
The total evaluation budget was $500$, so that there were $T=500$ rounds of evaluation for $N_{\mathrm{batch}}=1$, and $T=100$ rounds for $N_{\mathrm{batch}}=5$. Experimental details are explained in the Appendix.

We compared our one-step method (\texttt{Maha-BO}) with the two-step method based on the pseudo-inverse (\texttt{Maha-BO-pinv}), the randomized two-step method (\texttt{Maha-BO-random}), and previously proposed methods. For embedding-based methods, embedding dimension $d$ was set to $d_{\mathrm{true}}$ unless otherwise stated.

In addition, we compared our methods with several previously proposed methods.
On sequential problems, we compared our methods with the vanilla BO method with RBF kernel with ARD (\texttt{RBF-BO}), three random linear embedding methods: \texttt{REMBO}~\citep{Wang2016}, \texttt{HeSBO}~\citep{Nayebi19a}, \texttt{ALEBO}~\citep{Letham2020}, one learning linear embedding method: \texttt{SILBO}~\citep{Chen2020}, and a method without the linear embedding assumption: \texttt{TuRBO}~\citep{Eriksson2019}.
Since random embedding methods and TuRBO have special initialization strategies, initial points differed from those in other methods.
On batch problems, we compared our methods with the DPP-based BO method~\citep{Kathuria2016} with RBF kernel (\texttt{RBF-BO}), and \texttt{TuRBO}.

\begin{figure*}[ht]
\centering
\includegraphics[width=15cm]{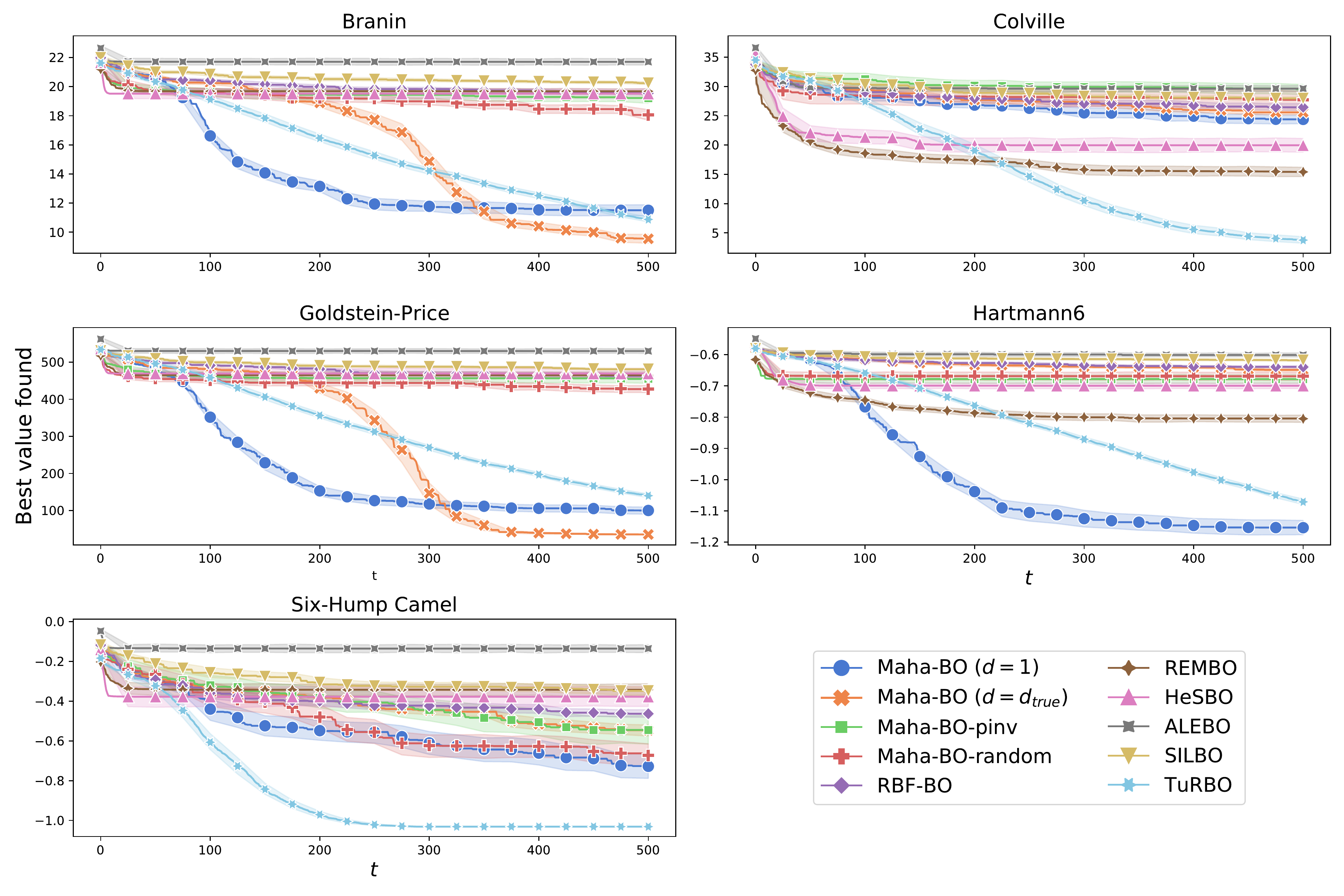}
\caption{Performance of sequential optimization ($N_{\mathrm{batch}}=1$) on five benchmark functions embedded in $D=1000$. The best value by each iteration is plotted with mean and 68\% confidence interval.}
\label{fig:benchmark_B1}
\end{figure*}

\begin{figure*}[ht]
\centering
\includegraphics[width=15cm]{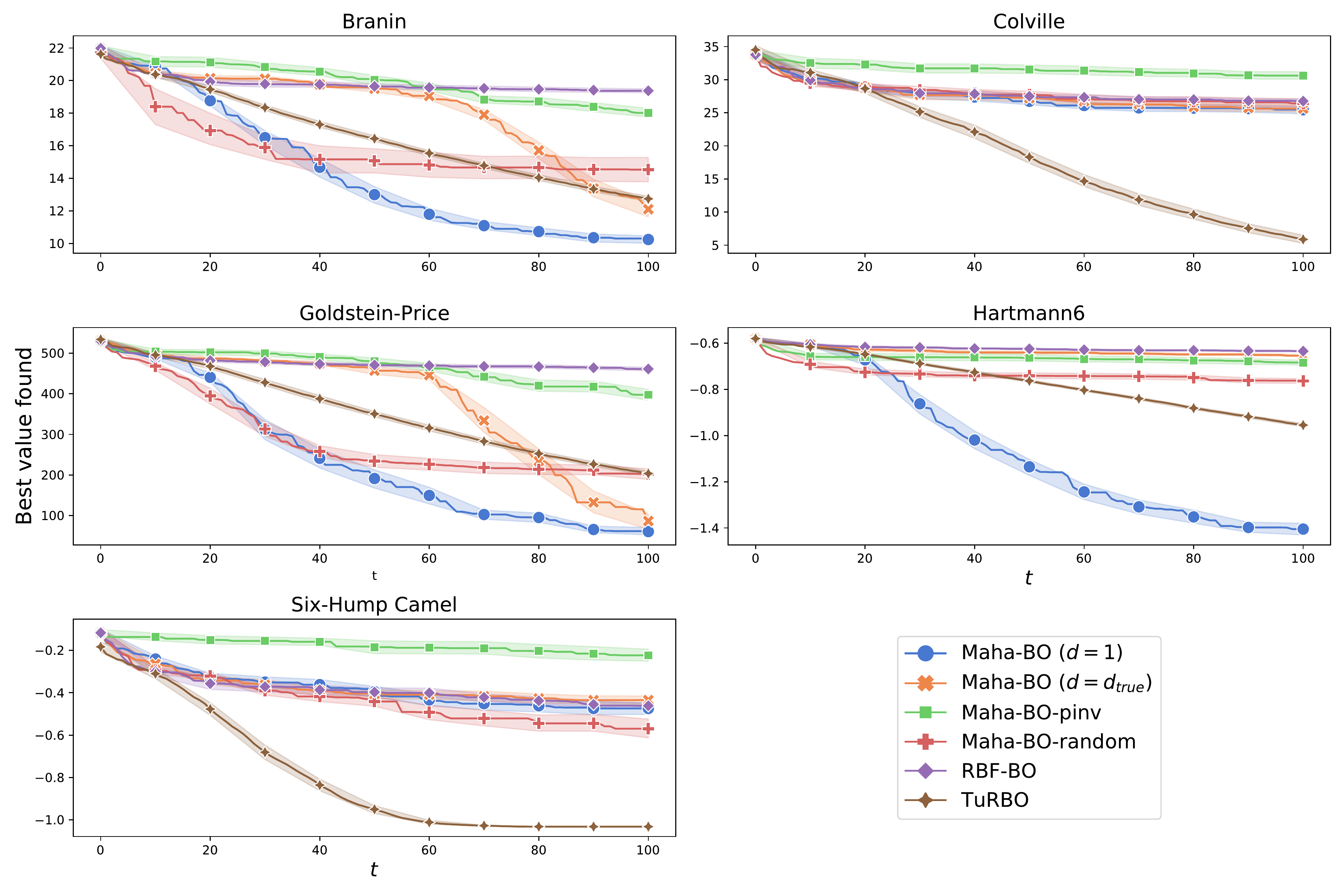}
\caption{Performance of batch optimization ($N_{\mathrm{batch}}=5$) on five benchmark functions embedded in $D=1000$. The best value by each iteration is plotted with mean and 68\% confidence interval.}
\label{fig:benchmark_B5}
\end{figure*}

The result is shown in Figure~\ref{fig:benchmark_B1} and \ref{fig:benchmark_B5}.
The one-step methods (\texttt{Maha-BO} ($d=1$) or \texttt{Maha-BO} ($d=d_\mathrm{true}$)) outperformed the other methods on three of five benchmarks in both sequential and batch optimization.
The two-step methods (\texttt{Maha-BO-pinv} and \texttt{Maha-BO-random}) perform poorly than the one-step methods. It demonstrates the significance of unbiased exploration as discussed in Section~\ref{sec:challenges}.

\texttt{ALEBO} and \texttt{SILBO} performed poorly in this experiment.
One possible reason is the difference in the true embedding matrix $A$ assumption. Previous studies often considered axis-aligned embedding, i.e., each column vector of $A$ has a value at only one dimension, or orthogonal embedding, i.e., column vectors are orthogonal to each other~\citep{Djolonga2013,Zhang2019,Chen2020,Letham2020}.
Their methods focus on the identification of effective subspace with an orthogonal basis. On the other hand, we do not impose any structural assumptions in learning the embedding matrix. Our method successfully extracted the non-orthogonal features from high-dimensional observations.

\texttt{TuRBO}, which does not use the linear embedding assumption, performed significantly better on Colville and Six-Hump Camel functions in both sequential and batch optimization settings. 
On these functions, all the methods based on learning the linear embedding showed similar poor performance. The methods based on the random embedding (\texttt{REMBO} and \texttt{HeSBO}) showed better performance, especially on Colville.
Thus, the result implies that these functions are intrinsically challenging to learn the low-dimensional structure from data.

Finally, we note that the performance of our proposed one-step method \texttt{Maha-BO} depends on the embedding dimensionality $d$. On Branin and Goldstein-Price functions, the method with $d=d_{\mathrm{true}}$ found a better solution than the method with $d=1$ in the sequential setting. However, the method with $d=1$ showed better performance in the other cases.

It contrasts the common assumption $d \geq d_{\mathrm{true}}$ on the random embedding methods. 
The assumption is required to increase the probability of the true optimizer falling into the search space $h(\mathcal{X})$~\citep{Wang2016}. As discussed in Section~\ref{sec:choice_of_embedding}, the learning embedding methods like ours do not rely on the probabilistic argument.

\begin{figure*}[ht]
\centering
\includegraphics[width=15cm]{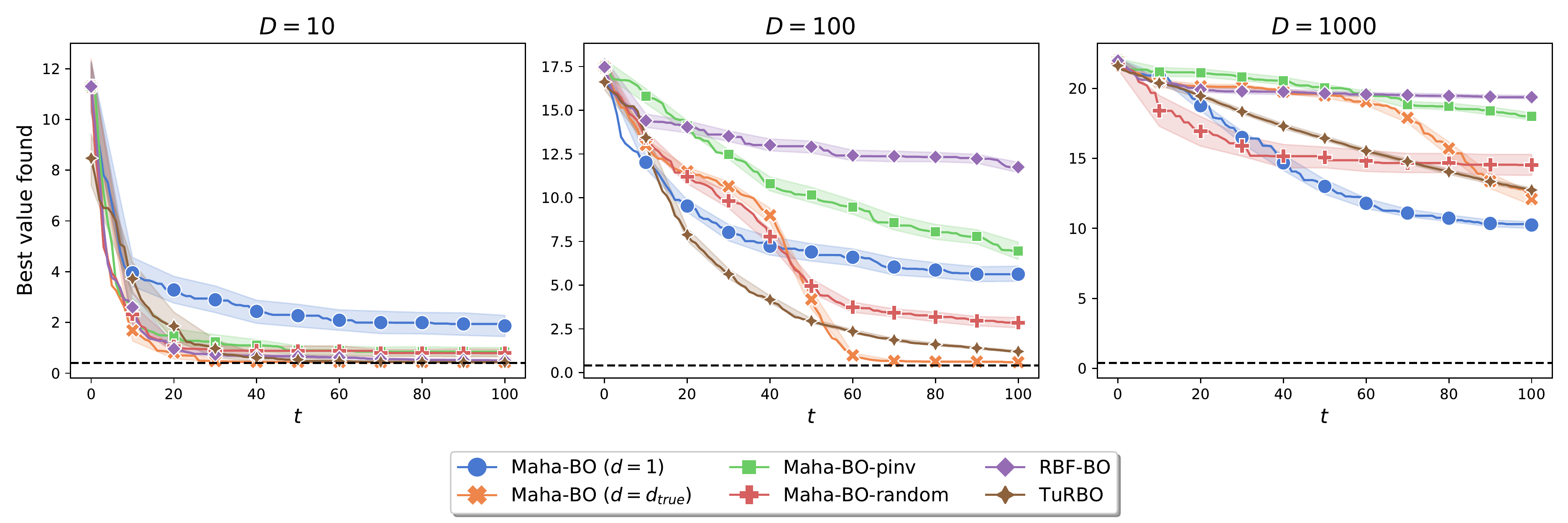}
\caption{Performance of batch optimization ($N_{\mathrm{batch}}=5$) on Branin functions embedded in $D=10, 100, 1000$. The best value by each iteration is plotted with mean and 68\% confidence interval. The dashed horizontal line indicates the minimum of the original Branin function.}
\label{fig:branin_d10_1000_B5}
\end{figure*}

One possible factor determining the best embedding dimension $d$ for our method is the difficulty of the optimization problem.
To investigate that point, we further examined the batch optimization of the Branin function varying $D$ from 10 to 1000.
The result is shown in Figure~\ref{fig:branin_d10_1000_B5}. As is evident from the final best value found, smaller $D$ makes the optimization problem easier.
When the problem is easy ($D=10$), most of the methods found the true optimum within the budget except for \texttt{Maha-BO} with $d=1$.
In contrast, \texttt{Maha-BO} with $d=1$ performed best for the difficult problem ($D=1000$).
The same method with $d=d_{\mathrm{true}} = 2$ begins to find good solutions later than the method with $d=1$. While it could surpass the method with $d=1$ and find the true optimal within the budget when $D=10,100$, it was too slow and could not exceed within the budget when $D=1000$.
Therefore, the result implies that the method with $d=1$ could quickly find an effective direction (one-dimensional space), but it is dominated by other methods when one more effective direction (two-dimensional space) comes into focus.

\subsection{Rover trajectory optimization}\label{sec:rover}
We also conducted more realistic experiments on rover trajectory optimization problem~\citep{Wang2018}. The goal is to optimize the location of 30 waypoints in 2D so that a rover can navigate from a start position to a goal position without colliding with obstacles.

First, we performed sequential and batch optimization experiments on this 60-dimensional function. From 10 initial points, we repeat $T=500$ rounds of evaluation 20 times.
We compare our method \texttt{Maha-BO} ($d=1, 2$) with the same existing methods as in the previous experiments. 
The results are shown in Figure~\ref{fig:rover_B1} and \ref{fig:rover_B5}~(Left). Since this function does not have an explicit linear embedding structure, it was not surprising that embedding-based methods performed poorly. \texttt{TuRBO} enjoyed the benefit of local modeling. Nevertheless, it is noteworthy that our methods with Mahalanobis kernel outperformed the RBF-kernel-based counterpart \texttt{RBF-BO}. It implies the existence of an intrinsic linear embedding structure of the function.

\begin{figure*}[ht]
\centering
\includegraphics[width=15cm]{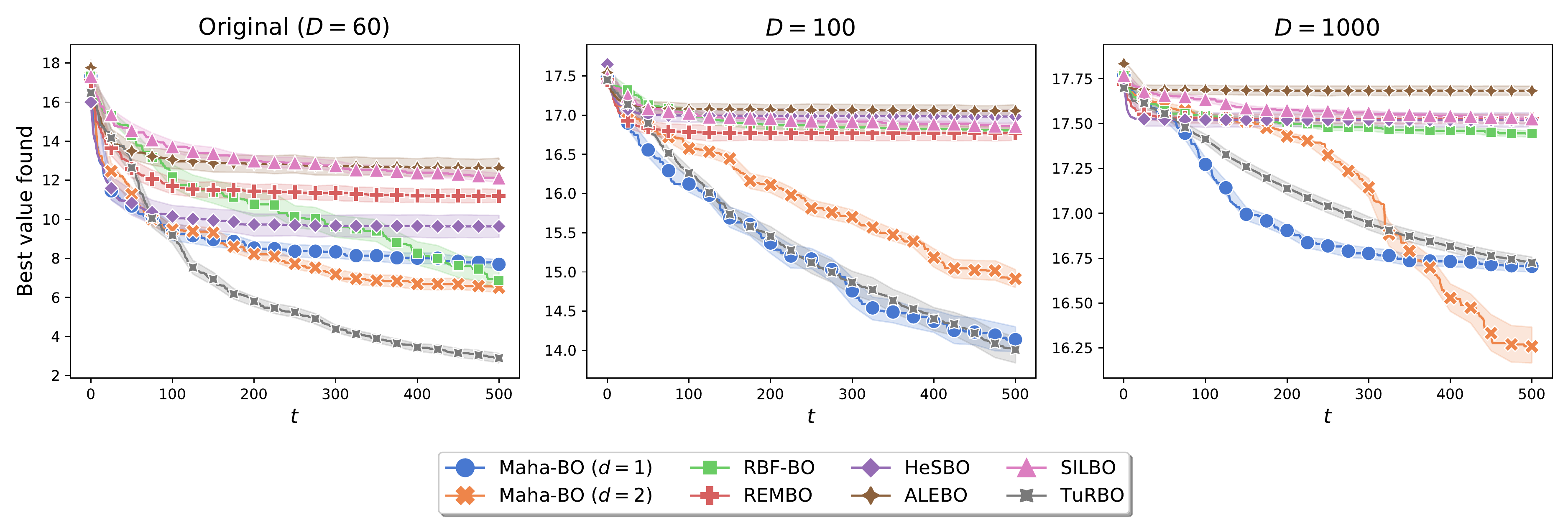}

\caption{Performance of sequential optimization ($N_{\mathrm{batch}}=1$) on the rover trajectory function (original $D=60$, embedded in $D=100$, and embedded in $D=1000$). The best value by each iteration is plotted with mean and 68\% confidence interval.}
\label{fig:rover_B1}
\end{figure*}

\begin{figure*}[ht]
\centering
\includegraphics[width=15cm]{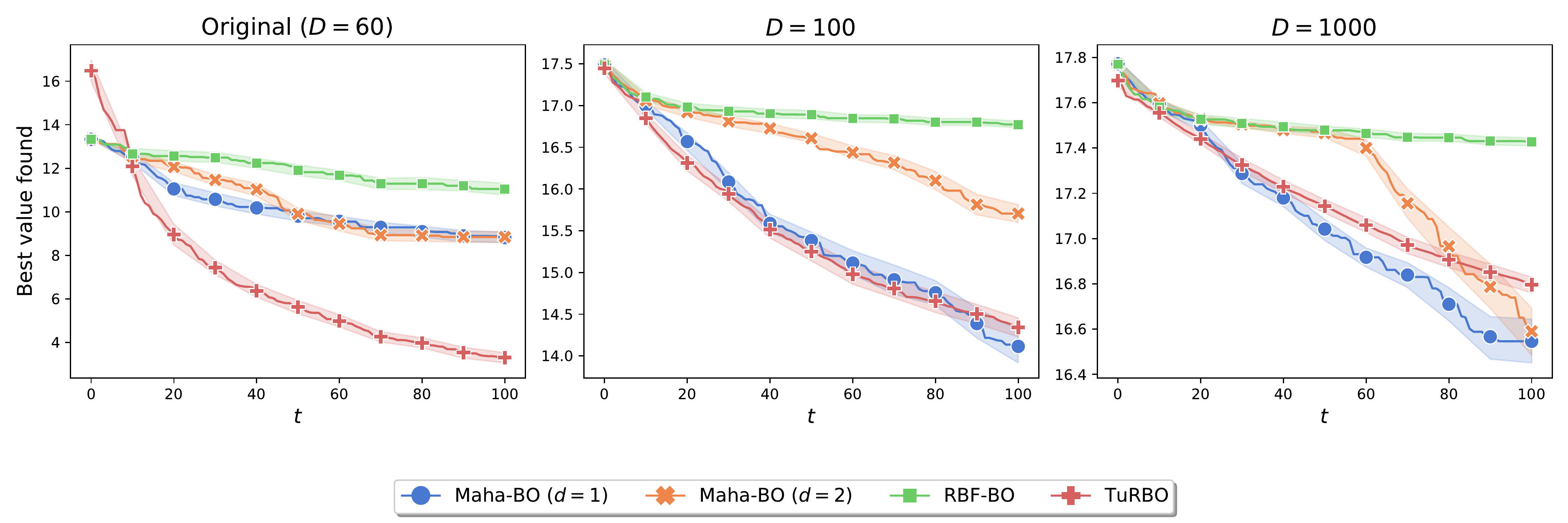}
\caption{Performance of batch optimization ($N_{\mathrm{batch}}=5$) on the rover trajectory function (original D = 60, embedded in D = 100, and embedded in D = 1000). The best value by each iteration is plotted with mean and 68\% confidence interval.}
\label{fig:rover_B5}
\end{figure*}

Additionally, as we did for the benchmark functions, we performed sequential and batch optimization experiments on the linearly projected version of this function. The results for $D=100$ and $1000$ are shown in Figure~\ref{fig:rover_B1} and \ref{fig:rover_B5}~(Middle, Right). Our methods outperformed the other methods in these situations.
As in the result for the Branin function, our method with embedding dimensionality $d=1$ quickly finds better solutions than $d=2$ in the early stages, whereas the method with $d=2$ could eventually find much better solutions in the $D=1000$ sequential setting.

\section{Conclusion}

We have considered the batch and high-dimensional optimization problem with linear embedding assumption. Our theoretical and empirical results show that the two-step method in high-dimensional Bayesian optimization—a step to select queries in a low-dimensional space followed by a step to reconstruct queries in the original high-dimensional space—has several drawbacks in terms of exploration. We have proposed a one-step method with Mahalanobis kernel and Determinantal point processes. It directly finds query points in the original high-dimensional space exploiting the low-dimensional structure. We have empirically verified that our method performs excellently on high-dimensional problems when the linear embedding assumption is satisfied.

Our method with one-dimensional embedding showed excellent performance in the early stage for most of the problems. Thus it is appropriate for cases where the quick finding of a suboptimal solution matters. However, as the number of available observations increases, the methods with higher embedding dimensions become superior. One future direction is to investigate the method by adaptively changing the embedding dimensionality.

In this paper, we have considered only the linear embedding assumption. Similar considerations should apply to the nonlinear embedding assumption, e.g., embedding by VAE~\citep{Moriconi2020}.
While the inverse of a linear mapping is just a linear subspace, the inverse image of a nonlinear mapping is difficult to handle. In~\cite{Moriconi2020}, the reconstruction mapping is learned by fitting the data. We believe that the reconstruction mapping should be re-examined in the light of unbiased and efficient exploration.


\acks{The research of the first author is supported by Epistra Inc.}


\newpage

\appendix

\section{Additional details of the experiment}
Here we provide additional information about the experiments. See the source code at \url{https://github.com/s-horiguchi/MahalanobisBatchBO} for exact implementation.

\subsection{Acquisition optimization}
In the one-step method, we ran L-BFGS-B from $m=5$ initial conditions. 
For sampling initial conditions, we use the heuristics that are implemented in BoTorch~\citep{balandat2020botorch}.

In the randomized two-step method (\texttt{Maha-BO-random}), we use the Algorithm~\ref{alg:random_reconstruction} in the main text. However, in some cases, sampling a good initial condition in the reconstruction step did not succeed in a reasonable time. Thus, we fell back to the pseudo-inverse mapping if the gradient descent could not find a convergent point in $\mathcal{X}$ after trying 100 initial conditions.

\subsection{DPP sampling}
We made a few modifications to the DPP sampling part from the DPP-EST-SAMPLE algorithm~\citep{Kathuria2016}.

First, previous studies only considered the discretized domain $\mathcal{X}$~\citep{Kathuria2016,Wang2017} since DPP sampling had been difficult over a continuous domain. However, in high-dimensional continuous domain problems, the usual lattice discretization becomes intractable since the size of the discretized domain grows exponentially with the dimension. We adopt a recently developed continuous $k$-DPP sampler by~\cite{pmlr-v97-rezaei19a}, which is also efficient in the high-dimensional continuous domain.

Second, we adopt a slightly different kernel for the $k$-DPP
\begin{equation*}
	L_t(x, x^\prime) = \kappa_{t,1}(x,x^\prime).
\end{equation*}
This kernel corresponds to the small noise limit $\sigma \rightarrow 0$ of the original kernel. The modification allows a more straightforward theoretical analysis.

Third, the relevant region is generalized as follows
\begin{equation*}
\mathcal{R}_t = \left\{ x\in\mathcal{X} ~\big\vert~ \alpha_t^{(-\lambda\beta_t)}(x) \leq \min_{x^\prime\in\mathcal{X}}\alpha_t^{(\beta_t)}(x^\prime) \right\}.
\end{equation*}
The parameter $\lambda \geq 1$ determines the size of the relevant region, which has been typically set to $2$~\citep{Contal2013,Kathuria2016}. In $k$-DPP sampling, the method by~\cite{pmlr-v97-rezaei19a} requires many uniform samples from the ground set, the relevant region. We use rejection sampling here. If the relevant region is too small, it is difficult to obtain enough samples. We gradually increase $\lambda$ from $2$ so that the rejection sampling would terminate after a finite number of steps.

\subsection{Other methods}
We employed the implementation on the Ax platform\footnote{\url{https://ax.dev/}} for \texttt{REMBO}, \texttt{HeSBO} and \texttt{ALEBO}, SILBO-TD implementation by the authors\footnote{\url{https://github.com/cjfcsjt/SILBO}}, and TuRBO\footnote{\url{https://github.com/uber-research/TuRBO}}.
Embedding dimension was $d=d_\mathrm{true}$ for \texttt{REMBO} and \texttt{SILBO} and $d=2d_\mathrm{true}$ for \texttt{HeSBO} and \texttt{ALEBO}. For \texttt{TuRBO}, the number of trust regions was set to $m=5$, and the number of initial points for each trust region was 10.
We implemented these methods with GPyTorch~\citep{gardner2018gpytorch}.

\subsection{High-dimensional objective functions}
Given a function $\tilde{f}$ on $d_\mathrm{true}$ dimensional box-shaped $\mathcal{Z}$ and arbitrary dimensionality $D > d_\mathrm{true}$, we constructed a high-dimensional objective function $f$ whose domain is a $D$ dimensional box $\mathcal{X}$ in the following way. 

First of all, without loss of generality, we can assume a normalized box $\mathcal{Z} = [-1, +1]^{d_\mathrm{true}}$. Let $\tilde{A} \in \mathbb{R}^{d_\mathrm{true} \times D}$, whose elements are sampled from i.i.d. standard Gaussian distributions. 
Then, we calculate $A \in \mathbb{R}^{d_\mathrm{true} \times D}$ as follows:
\begin{equation*}
	A_{ij} = \frac{\tilde{A}_{ij}}{\sum_{k=1}^D |\tilde{A}_{ik}|}.
\end{equation*}
Finally, we set the domain $\mathcal{X} = [-1, +1]^D$. The function $f$ is evaluated as $f(x) = \tilde{f}(Ax)$.

Using this $A$, $\tilde{f}$ would be evaluated only on $\mathcal{Z}$ not on the outside of $\mathcal{Z}$ as long as the input $x$ is in $\mathcal{X}$. To show that, it is sufficient to prove that for any $x\in\mathcal{X}= [-1, +1]^D$, $z = Ax$ is a member of $\mathcal{Z}$. It follows from a simple calculation
\begin{equation*}
\begin{aligned}
	\underset{x\in \mathcal{X}}{\max}~z_{i} 
	&=\sum _{j=1}^{D}\max( A_{ij} ,\ -A_{ij}) =\sum _{j=1}^{D}| A_{ij}| = 1,\\
	\underset{x\in \mathcal{X}}{\min}~z_{i} 
	&=\sum _{j=1}^{D}\min( A_{ij} ,\ -A_{ij}) =-\sum _{j=1}^{D}| A_{ij}| =-1.	
\end{aligned}
\end{equation*}

Note that the constructed function $f$ has larger minimum value $ \min_{x \in \mathcal{X}}~f(x) \geq \min_{z\in\mathcal{Z}} \tilde{f}(z)$. 
The original optimum in $\mathcal{Z}$ could have no corresponding points in $\mathcal{X}$ since the image of $\mathcal{X}$ by $g(x)=Ax$ is a subset of $\mathcal{Z}$.

\subsection{Detailed results on the experiments}
We show the best values after $T$ rounds of optimization for five benchmark functions are shown in Table~\ref{tab:benchmark_B1_5}.

\begin{landscape}
\begin{table*}[ht]
\centering
{\small
\begin{tabular}{lrrrrr}
\toprule
\multicolumn{1}{l}{} & \multicolumn{5}{c}{$N_{\mathrm{batch}}=1$}\\
\multicolumn{1}{l}{Benchmark} & \multicolumn{1}{c}{Branin} & \multicolumn{1}{c}{Colville} & \multicolumn{1}{c}{Goldstein-Price} & \multicolumn{1}{c}{Hartmann6} & \multicolumn{1}{c}{Six-Hump} \\
\midrule
\texttt{Maha-BO} ($d=1$)  &11.50 $\pm$ 0.38  &  24.39 $\pm$ 0.79  & 100.16 $\pm~~$ 9.78  & \textbf{-1.15 $\pm$ 0.02}  & -0.73 $\pm$ 0.06  \\
\texttt{Maha-BO} ($d=d_\mathrm{true}$)  &\textbf{9.55 $\pm$ 0.30}  &  25.59 $\pm$ 0.45  & \textbf{35.57 $\pm~~$ 1.41}  & -0.65 $\pm$ 0.00  & -0.55 $\pm$ 0.03  \\
\texttt{Maha-BO-pinv}  &19.19 $\pm$ 0.34  &  29.50 $\pm$ 0.89  & 454.79 $\pm$ 12.74  & -0.68 $\pm$ 0.01  & -0.54 $\pm$ 0.07  \\
\texttt{Maha-BO-random}  &18.06 $\pm$ 0.37  &  27.75 $\pm$ 1.54  & 427.24 $\pm~~$ 9.52  & -0.67 $\pm$ 0.01  & -0.67 $\pm$ 0.06  \\
\texttt{RBF-BO}  &19.54 $\pm$ 0.11  &  26.45 $\pm$ 0.60  & 461.91 $\pm~~$ 5.50  & -0.64 $\pm$ 0.00  & -0.46 $\pm$ 0.02  \\
\texttt{REMBO}  &19.68 $\pm$ 0.18  &  15.41 $\pm$ 0.85  & 464.89 $\pm~~$ 9.44  & -0.80 $\pm$ 0.01  & -0.34 $\pm$ 0.03  \\
\texttt{HeSBO}  &19.52 $\pm$ 0.34  &  19.97 $\pm$ 1.23  & 469.22 $\pm$ 13.86  & -0.70 $\pm$ 0.01  & -0.38 $\pm$ 0.05  \\
\texttt{ALEBO}  &21.70 $\pm$ 0.23  &  29.64 $\pm$ 0.53  & 529.92 $\pm~~$ 6.43  & -0.60 $\pm$ 0.01  & -0.14 $\pm$ 0.02  \\
\texttt{SILBO}  &20.24 $\pm$ 0.10  &  27.96 $\pm$ 0.33  & 480.90 $\pm~~$ 2.63  & -0.62 $\pm$ 0.00  & -0.35 $\pm$ 0.02  \\
\texttt{TuRBO}  &10.88 $\pm$ 0.18  &  \textbf{3.76 $\pm$ 0.58}  & 139.98 $\pm~~$ 5.07  & -1.07 $\pm$ 0.01  & \textbf{-1.03 $\pm$ 0.00}  \\
\midrule
\multicolumn{1}{l}{} & \multicolumn{5}{c}{$N_{\mathrm{batch}}=5$}\\
\midrule
\texttt{Maha-BO} ($d=1$)  &\textbf{10.24 $\pm$ 0.24}  &  25.45 $\pm$ 0.57  & \textbf{60.29 $\pm~~$ 8.92} & \textbf{-1.41 $\pm$ 0.03}  & -0.47 $\pm$ 0.03  \\
\texttt{Maha-BO} ($d=d_\mathrm{true}$)  &12.10 $\pm$ 0.48  &  25.63 $\pm$ 0.53  & 86.34 $\pm$ 20.56  & -0.65 $\pm$ 0.00  & -0.43 $\pm$ 0.02  \\
\texttt{Maha-BO-pinv}  &18.02 $\pm$ 0.27  &  30.62 $\pm$ 0.59  & 397.62 $\pm$ 13.22  & -0.68 $\pm$ 0.01  & -0.22 $\pm$ 0.03  \\
\texttt{Maha-BO-random}  &14.53 $\pm$ 0.74  &  26.41 $\pm$ 0.69  & 202.17 $\pm$ 14.18  & -0.76 $\pm$ 0.01  & -0.57 $\pm$ 0.04  \\
\texttt{RBF-BO}  &19.37 $\pm$ 0.12  &  26.75 $\pm$ 0.34  & 461.26 $\pm~~$ 4.41  & -0.63 $\pm$ 0.00  & -0.46 $\pm$ 0.01  \\
\texttt{TuRBO}  &12.74 $\pm$ 0.16  &  \textbf{5.86 $\pm$ 0.62}  & 203.47 $\pm~~$ 5.28  & -0.95 $\pm$ 0.01  & \textbf{-1.03 $\pm$ 0.00}  \\
\bottomrule
\end{tabular}
}
\caption{Performance on five benchmark functions embedded in $D=1000$.
The mean and standard error of the best values are shown.}
\label{tab:benchmark_B1_5}
\end{table*}
\end{landscape}

\subsection{Running time}

We calculated the average running time of each method against the Branin problem with $D=10,100, 1000$ and summarized it in Table~\ref{tab:benchmark_time}. The average running of time of \texttt{Maha-BO} ($d=1$ or $d_{\mathrm{true}}$) was smaller than or comparable with that of \texttt{RBF-DPP}. Our method was slower than \texttt{REMBO}, \texttt{HeSBO}, \texttt{SILBO}, and \texttt{TuRBO}, but faster than \texttt{ALEBO} in this experiment. Also, the running time did not grow exponentially with respect to the dimensionality $D$.

\begin{table*}[ht]
\centering
{\small
\begin{tabular}{lrrr}
\toprule
\multicolumn{1}{l}{} & \multicolumn{3}{c}{$N_{\mathrm{batch}}=1$}\\
\multicolumn{1}{l}{Benchmark} & \multicolumn{1}{c}{Branin $D=10$} & \multicolumn{1}{c}{Branin $D=100$} & \multicolumn{1}{c}{Branin $D=1000$}\\
\midrule
\texttt{Maha-BO} ($d=1$)  &51.17 $\pm~~$ 4.14  &  110.65 $\pm~~$ 0.91  & 280.23 $\pm~~$ 4.32  \\
\texttt{Maha-BO} ($d=d_{\mathrm{true}}$)  &120.72 $\pm~~$ 1.35  &  116.58 $\pm~~$ 0.62  & 310.47 $\pm~~$ 2.90  \\
\texttt{Maha-BO-pinv}  &89.30 $\pm~~$ 1.46  &  108.85 $\pm~~$ 1.00  & 259.81 $\pm~~$ 3.71  \\
\texttt{Maha-BO-random}  &127.43 $\pm~~$ 4.09  &  113.49 $\pm~~$ 1.28  & 211.69 $\pm~~$ 10.92  \\
\texttt{RBF-BO}  &747.88 $\pm$ 50.03  &  266.06 $\pm~~$ 1.03  & 388.75 $\pm~~$ 1.09  \\
\texttt{REMBO}  &7.98 $\pm~~$ 0.32  &  9.95 $\pm~~$ 0.19  & 42.44 $\pm~~$ 0.36  \\
\texttt{HeSBO}  &23.62 $\pm~~$ 2.45  &  29.09 $\pm~~$ 1.40  & 108.35 $\pm~~$ 3.58  \\
\texttt{ALEBO}  &314.60 $\pm$ 52.00  &  319.58 $\pm$ 58.11  & 464.65 $\pm$ 34.23  \\
\texttt{SILBO}  &7.80 $\pm~~$ 0.51  &  12.16 $\pm~~$ 0.14  & 55.54 $\pm~~$ 0.70  \\
\texttt{TuRBO}  &2.69 $\pm~~$ 0.03  &  19.35 $\pm~~$ 0.34  & 625.77 $\pm$ 14.95  \\
\midrule
\multicolumn{1}{l}{} & \multicolumn{3}{c}{$N_{\mathrm{batch}}=5$}\\
\midrule
\texttt{Maha-BO} ($d=1$)  &32.58 $\pm~~$ 2.74  &  172.02 $\pm$ 14.04  & 204.28 $\pm$ 2.06  \\
\texttt{Maha-BO} ($d=d_{\mathrm{true}}$)  &780.87 $\pm$ 32.57  &  183.42 $\pm$ 12.61  & 346.49 $\pm$ 1.57  \\
\texttt{Maha-BO-pinv}  &1577.94 $\pm$ 78.21  &  37.08 $\pm~~$ 2.51  & 285.20 $\pm$ 2.03  \\
\texttt{Maha-BO-random}  &754.72 $\pm$ 25.49  &  132.06 $\pm~~$ 8.89  & 353.83 $\pm$ 9.19  \\
\texttt{RBF-BO}  &381.11 $\pm$ 18.27  &  99.00 $\pm~~$ 5.02  & 259.14 $\pm$ 0.68  \\
\texttt{TuRBO}  &1.54 $\pm~~$ 0.04  &  88.98 $\pm~~$ 1.42  & 108.57 $\pm$ 3.99  \\
\bottomrule
\end{tabular}
}
\caption{Running time per trial in sequential and batch optimization of Branin function. $T=500$, $D=10,100,1000$. Means and standard errors of the times in minutes are shown.}
\label{tab:benchmark_time}
\end{table*}

\vskip 0.2in
\bibliography{jmlr2022}

\end{document}